\documentclass[table]{article}

\usepackage{preprint, times}

\usepackage[utf8]{inputenc} 
\usepackage[T1]{fontenc}    
\usepackage{hyperref}       
\usepackage{url}            
\usepackage{booktabs}       
\usepackage{amsfonts}       
\usepackage{nicefrac}       
\usepackage{microtype}      
\usepackage{lipsum}		
\usepackage{graphicx}
\usepackage{natbib}
\usepackage{doi}

\usepackage{hyperref}
\usepackage{url}
\usepackage{graphicx}
\usepackage{longtable}
\usepackage{makecell}
\usepackage{array}
\usepackage{subfigure}
\usepackage{wrapfig}
\usepackage{caption}
\usepackage{subcaption} 
\usepackage{xcolor}
\usepackage{mdframed}
\usepackage{footmisc}
\usepackage{amsmath}
\usepackage{amsthm}
\usepackage{multirow}

\definecolor{bonus_green}{RGB}{0,100,0}
\definecolor{drop_red}{RGB}{180,0,0}
\definecolor{bg}{RGB}{176,226,255}
\newcommand{\bbonus}[1]{{\textcolor{bonus_green}{$^{\uparrow#1}$}}}
\newcommand{\ddrop}[1]{{\textcolor{drop_red}{$^{\downarrow#1}$}}}

\newtheorem{theorem}{Theorem}

\title{On Designing Effective RL Reward at Training Time for LLM Reasoning}

\author{Jiaxuan Gao$^{1,2,*}$
	\quad
	Shusheng Xu$^{1,2,*}$ \quad Wenjie Ye$^3$ \quad Weilin Liu$^3$ \quad Chuyi He$^3$\\
\textbf{Wei Fu$^{1,2}$ \quad Zhiyu Mei$^{1,2}$ \quad Guangju Wang$^3$ \quad Yi Wu$^{1,2,3,\dag}$
} \\
$^1$ Institute for Interdisciplinary Information Sciences, Tsinghua University \\
$^2$ Shanghai Qi Zhi Institute \quad $^3$OpenPsi Inc.\\
\texttt{\{samjia2000, xssstory, jxwuyi\}@gmail.com} \\
}

\definecolor{bonus}{rgb}{0.0, 0.8, 0.0}

\begin{document}
\renewcommand{\thefootnote}{\fnsymbol{footnote}}
\footnotetext[1]{Equal Contribution}
\footnotetext[2]{Corresponding Author}
\renewcommand{\thefootnote}{\arabic{footnote}}

\maketitle

\begin{abstract}
Reward models have been increasingly critical for improving the reasoning capability of LLMs. Existing research has shown that a well-trained reward model can substantially improve model performances \emph{at inference time} via search or best-of-N votes. 
However, the potential of reward models during \emph{RL training time} still remains largely under-explored. 
It is currently unclear whether these reward models can provide additional training signals to RL training that uses sparse success rewards, which verify the correctness of solutions.
In this work, we evaluate popular reward models for RL training, including the Outcome-supervised Reward Model (ORM) and the Process-supervised Reward Model (PRM), and train a collection of LLMs for math problems using RL by combining these learned rewards with success rewards. Surprisingly, even though these learned reward models have strong inference-time performances, they may \emph{NOT} help or even hurt RL \emph{training}, producing worse performances than LLMs trained with the success reward only. 
Our analysis reveals that an LLM can receive high rewards from some of these reward models by repeating correct but unnecessary reasoning steps, leading to a severe reward hacking issue for RL training. 
Therefore, we introduce two novel reward refinement techniques, including \textbf{\emph{Clipping}} and \textbf{\emph{Delta}}. The key idea is to ensure the accumulative reward of any reasoning trajectory is upper-bounded to keep a learned reward model effective without being exploited. 
We evaluate our techniques with multiple reward models over a set of 1.5B and 7B LLMs on MATH and GSM8K benchmarks, where both \textbf{\emph{Clipping}} and \textbf{\emph{Delta}} consistently stabilize RL training. 
Finally, we also demonstrate that with a carefully designed reward function, pure RL training without any additional supervised tuning can further improve all the evaluated LLMs, including the state-of-the-art 7B LLM Qwen2.5-Math-7B-Instruct on MATH and GSM8K benchmarks.
\end{abstract}

\section{Introduction}

There is a recent trend to improve the reasoning ability of LLMs with learned reward models~\citep{lightman2023let,wang2024math,yu2024ovm,zhang2024generative,lee2024token,yang2024qwen2,luo2024improve,chen2024improving,havrilla2024teaching,shao2024deepseekmath,uesato2022solving}. 
Recent research has been focusing on guiding search processes during inference~\citep{lightman2023let,snell2024scaling,wang2024math}, with two main categories of reward models: Outcome-supervised Reward Model (ORM)~\citep{cobbe2021training, yu2024ovm} and Process-supervised Reward Model (PRM)~\citep{lightman2023let,wang2024math,luo2024improve}.
ORM generates \emph{outcome rewards} that estimate the \emph{success rewards}, which evaluate the correctness of generated answers, enabling the selection of the most reliable answer from a pool of generated candidates.
By contrast, PRM is trained to distinguish correct reasoning steps from incorrect ones and can provide step-level \emph{process rewards} for 
search algorithms like Monte-Carlo Tree Search ~\citep{chen2024alphamath} and beam search~\citep{snell2024scaling}.

However, the potential of reward models in RL training for LLM reasoning is not yet fully explored. The most straightforward method for RL training in reasoning tasks is to optimize the success rewards. Some prior works further try the integration of a reward model into RL training
~\citep{havrilla2024teaching, wang2024math, shao2024deepseekmath}. \cite{havrilla2024teaching} finds that PPO training with a reward model only results in performance degeneration.
In addition, some powerful LLMs that exhibit strong reasoning abilities such as the Qwen2.5-Math family~\citep{yang2024qwen2} and DeepseekMath-7B-RL~\citep{shao2024deepseekmath} adopt RL training with reward models as a part of their overall training process for mathematical reasoning.
However, due to a lack of detailed analysis on the reward models, it remains unclear whether the reward models can provide additional training signals beyond what the success rewards offer for LLM reasoning.

In this work, we evaluate popular reward models, including ORM and PRM, as RL rewards on the challenging mathematical reasoning benchmark MATH~\citep{MATH} and GSM8K~\citep{GSM8K} by using PPO as the RL algorithm~\citep{schulman2017proximal}. Surprisingly, we find that these reward models may not enhance RL training or even lead to performance degradation, yielding even worse results than LLMs trained with a sparse success reward only. We observe that outcome rewards consistently achieve similar training results as success rewards. We hypothesize that outcome rewards may not be beneficial at \emph{training} time since a more accurate success reward is accessible.  For PRM, we perform an in-depth analysis of the RL training process and identify a severe \emph{reward hacking issue}~\citep{casper2023open,rame2024warm,singhal2023long}.
Reward hacking manifests in the form of generating numerous correct but unnecessary reasoning steps.
Through RL training, an LLM could exploit the 
PRM to achieve an excessively high by repeated generating simple reasoning steps that may not contribute to solving the problem, leading to a completely undesirable LLM behavior with poor reasoning accuracy. 

To tackle these challenges, we propose two novel techniques, i.e., \textbf{\emph{Clip}} and \textbf{\emph{Delta}}, 
which refines the process rewards for effective RL training.
In particular, the \textbf{\emph{Clip}} mechanism bounds rewards to an upper threshold so that RL training can focus on reducing erroneous reasoning steps.
The \textbf{\emph{Delta}} mechanism maintains a bounded objective by subtracting the rewards between two adjacent steps, discouraging trivial repetition patterns to achieve a high return and improving training stability.
Evaluation of these two techniques on synthetic reasoning trajectories demonstrates that they can mitigate the reward hacking issue consistently.
Finally, we conduct full RL training on a set of advanced 1.5B and 7B LLMs from the Qwen2 and Qwen2.5 families~\citep{qwen2,yang2024qwen2} with different reward models. Our experiment results show that our proposed techniques effectively stabilize RL training. Moreover, with a carefully crafted reward, RL training can 
improve all the evaluated LLMs, including the state-of-the-art 7B LLM Qwen2.5-Math-7B-Instruct on the challenging MATH and GSM8K~\citep{MATH,GSM8K} benchmarks.

\section{Related Work}

\textbf{Reinforcement Learning for LLMs.} 
In RLHF, Reinforcement learning algorithms can effectively fine-tune LLMs to align with the preference of humans~\citep{dong2023raft, rafailov2024direct, ouyang2022training, xu2024dpo,schulman2017proximal}, to improve the reasoning ability~\citep{shao2024deepseekmath,yang2024qwen2} and coding skills~\citep{wang2023mathcoder,guo2024deepseek}. PPO is the most widely used among the popular RL algorithms due to its robust performance across various domains~\citep{ouyang2022training, xu2024dpo}. \cite{xu2024dpo} investigates the implementation details of PPO for dialogue tasks and coding tasks, revealing batch size as a critical factor for improving PPO performance in reinforcement learning from human feedback (RLHF). Our work addresses the challenge of designing RL rewards for LLM reasoning.

\textbf{Reward Learning for LLMs.} Learned reward models are widely adopted in RLHF to align LLMs with human preferences~\citep{dong2023raft, rafailov2024direct, ouyang2022training}. In RLHF, reward models are trained on binary preference datasets collected from human annotators, following the Bradley-Terry model~\citep{bradley1952rank}. In reasoning tasks involving reliable solution checkers, two main approaches are the Outcome-supervised Reward Model (ORM)~\citep{cobbe2021training, yu2024ovm} and the Process-supervised Reward Model (PRM)~\citep{lightman2023let,wang2024math,luo2024improve}. An ORM predicts the likelihood that the final answer of a solution prefix would be correct. A PRM estimates whether the steps so far are correct for each reasoning step. Through training over extensive corpora, reward models are able to evaluate solution quality. Despite the successful applications of reward models, \emph{reward hacking} is a broadly observed issue in learned reward models ~\citep{skalse2022defining,singhal2023long,casper2023open}. Through RL training, the LLM may learn to generate high-reward outputs that could not fulfill the intended objectives. Several approaches have been proposed to tackle the reward hacking issue, including disentangling the length aspect of reward modeling~\citep{chen2024odin, shen2023loose}, reward ensemble~\citep{eisenstein2023helping,rame2024warm}, length penalty~\citep{singhal2023long}, length normalization~\citep{meng2024simpo}, and various PPO implementation tricks~\citep{singhal2023long,zheng2023secrets}. In this work, we investigate the reward hacking issue for reasoning tasks when combining learned rewards and success rewards in RL training.

\textbf{Improving Reasoning Ability of LLMs.} 
To improve the reasoning ability of LLMs, prior works have focused on several different aspects, including pre-training~\citep{yang2024qwen2, achiam2023gpt, anil2023palm}, prompting~\citep{toh2023veritymath, yuan2024advancing, wu2024large}, search during inference-time~\citep{lightman2023let,wang2024math,yu2024ovm,zhang2024generative,yang2024qwen2,luo2024improve,chen2024improving}, and fine-tuning~\citep{wang2024math,shao2024deepseekmath,yang2024qwen2,shah2024ai,tang2024mathscale,yu2023metamath}. Pre-training methods focus on enriching the data distribution to cover a large amount of rationals and pre-training the LLM over the dataset. The prompting methods elicit the reasoning ability of LLMs through dedicated prompting strategies and automatic agent frameworks. Inference-time search utilizes learned reward models to guide the selection of promising solutions. PRM and ORM could be combined with different search strategies such as Best-of-N, Monte-Carlo Tree Search~\citep{chen2024alphamath}, and Beam Search~\citep{snell2024scaling}. Finally, fine-tuning methods include training the LLM on high-quality question-answer data~\citep{yu2023metamath,shah2024ai,yue2023mammoth} and optimizing the reasoning ability with reinforcement learning~\citep{yang2024qwen2,shao2024deepseekmath,wang2024math}. In this work, we study how to effectively combine dense and sparse rewards in  RL training for reasoning tasks.

\section{Preliminary}
\label{sec:prelim}


\textbf{Language Model.} 
An LLM is represented as a policy $\pi_\theta(s|q)$ parameterized by $\theta$.
In reasoning tasks, $\pi_\theta$ generates a solution $s$ given a question $q$. In addition to the question, $q$ usually also contains a prompt to elicit chain-of-thought reasoning. The solution $s$ is structured with a list of reasoning steps and thus can be viewed from two perspectives, including tokens and steps. From the perspective of tokens, $s$ consists of $T$ tokens, $s=(s_1,s_2,\cdots,s_T)$. From the perspective of steps, $s$ consists of $K$ reasoning steps, $s=(s^{(1)},s^{(2)},\cdots,s^{(K)})$ where $s^{(k)}$ denotes the $k$-th reasoning step. For convenience, we use $p^{(k)}=(s^{(1)},s^{(2)},\cdots,s^{(k)})$ to denote the solution prefix up to the $k$-th step. In practice, reasoning steps can be parsed with rule-based detectors, enforcing strict output formats, or special tokens~\citep{chen2024alphamath,wang2024math,lightman2023let}.

\textbf{Reward Modeling.}
In RLHF, the reward models are usually trained with binary preferences~\citep{bradley1952rank}. In reasoning tasks where the correctness of solutions is accessible, reward models can be trained under the supervision of such ground-truth correctness.
In reasoning tasks, two primary methods for reward modeling are the Process-supervised Reward Model (PRM) and the Outcome-supervised Reward Model(ORM). 

Given a question $q$ and a prefix $s_{1:t}$, an ORM estimates the likelihood the prefix would lead to a correct answer. A standard approach to train an ORM is by first sampling solutions for questions from a dataset with an LLM and then labeling the correctness of each solution. The ORM $r_{\text{outcome}}$ is then trained with the following objective,
\begin{align*}
\mathcal L_{\text{ORM}}=\mathbb E_{q, s\sim \mathcal D}\left[\sum_{t=1}^{T}\text{Loss}(\text{Correct}(q, s), r_{\text{outcome}}(q,s_{1:t}))\right]
\end{align*}

where $\text{Correct}(q, s)$ is a binary value indicating the correctness of solution $s$, $t$ enumerates each token of the solution $s$, and $\text{Loss}$ denotes the loss function. In practice, the loss function could be binary cross-entropy loss or square-error loss, and we can choose to train ORM on the full sequence or only the last token.

In contrast, Process-supervised Reward Model (PRM) estimates the correctness of individual reasoning steps. PRM is trained with the following objective,
\begin{align*}
\mathcal L_{\text{PRM}}=\mathbb E_{q, p^{(k)}, y_k\sim \mathcal D}\left[\text{Loss}(y_k, r_{\text{process}}(q,p^{(k)}))\right]
\end{align*}
where $y_k$ is the label for the partial solution $p^{(k)}$ and $\text{Loss}$ is the loss function. In practice, binary cross entropy loss is usually adopted. Prior works have investigated several ways to annotate the process labels, including human annotators~\citep{lightman2023let} and automatic annotation with LLMs~\citep{wang2024math,luo2024improve}. 

\textbf{Reinforcement Learning for LLM 
 Reasoning.} 
 We assume access to the correctness of a solution during training. We use $\text{Correct}(q, s)$ to indicate the correctness of solution $s$ to question $q$, which is also referred to as the \emph{success reward} for RL training. An LLM can be fine-tuned to optimize the success reward by using Reinforcement Learning with Kullback-Leibler divergence, 
\begin{align}
\label{eq:rl}
J_r(\pi_\theta)=\mathbb E_{q\sim \mathcal D, s\sim\pi_\theta}\left[\text{Correct}(q,s) - \beta\log\frac{\pi_\theta(s|q)}{\pi_{ref}(s|q)}\right]
\end{align}
where $\pi_{ref}$ is the reference model for regularizing $\pi_\theta$.
Optimizing the success reward only provides a sparse training signal because the reward is provided at the end of the sequence. Alternatively, we can also combine dense rewards with the success reward for more fine-grained training signals. The RL objective with dense rewards becomes,
\begin{align}
\label{eq:rl-dense}
J_r(\pi_\theta)=\mathbb E_{q\sim \mathcal D, s\sim\pi_\theta}\left[\alpha\cdot\sum_{t=1}^{|s|}r(q,s_{1:t}) + \text{Correct}(q,s) - \beta\log\frac{\pi_\theta(s|q)}{\pi_{ref}(s|q)}\right]
\end{align}
where $r$ denotes the dense reward and $\alpha$ is a coefficient for the dense reward. For example, a PRM $r_{\text{process}}$ can provide dense feedback at the end of reasoning steps, formally represented as $r(q,p^{(k)})=r_{\text{process}}(q,p^{(k)})$ for any partial solution $p^{(k)}$.

\section{RL Reward for LLM Reasoning}

In this section, we conduct a systematic study on reward design to aid LLM in learning better reasoning skills through RL training. We follow the RL objective with dense rewards in Eq.~(\ref{eq:rl-dense}) and specifically focus on the effective design of dense rewards.
As discussed in Sec.~\ref{sec:prelim}, the ground-truth correctness, $\text{Correct}(p, s)$, serves to provide the sparse rewards, and the dense rewards could be provided by a reward model.

\subsection{Evaluating RL Training with Learned Reward Models}
\label{sec:common}

We first consider two straightforward approaches to apply ORM and PRM to provide rewards in addition to success rewards for RL training. Formally, we consider the following rewards,
\begin{itemize}
\item \textbf{Solution-Level Outcome Reward (OR):} In the RL training process of \cite{yang2024qwen2}, an ORM provides an estimation of correctness as reward shaping. Note that this is not the case for dense rewards since ORM only produces rewards at the end of the sequence. For a question $q$ and a solution $s$,
\begin{align}
\label{eq:or}
r(q,s)=r_{\text{outcome}}(q,s)
\end{align} 
\item \textbf{Step-Level Process Reward (PR):}  A PRM can provide step-level feedback for RL training. For any solution prefix $p^{(k)}$, dense rewards are the rewards outputted by a PRM,
\begin{align}
\label{eq:pr}
r(q,p^{(k)})=r_{\text{process}}(q, p^{(k)})
\end{align}
\end{itemize}

\label{sec:reward-hacking}

\paragraph{Experiment Setup.} 
We carry out our study on the challenging mathematical reasoning benchmark, MATH~\citep{MATH}. We use PPO as the RL algorithm and Qwen2-1.5B-Instruct~\citep{qwen2} as the base model. For ORM, we sample solutions with the base model and train ORM with binary cross-entropy loss. For PRM, we follow \cite{wang2024math} to generate process labels with automatic annotation\footnote{Implementation details can be found in Sec.~\ref{sec:exp}}. The ORM and PRM both use Qwen2-1.5B-Instruct as the base model.

\begin{figure}[h]
    \centering
    \subfigure[Greedy Accuracy]{
        \label{fig:or-pr}
        \includegraphics[height=35mm]{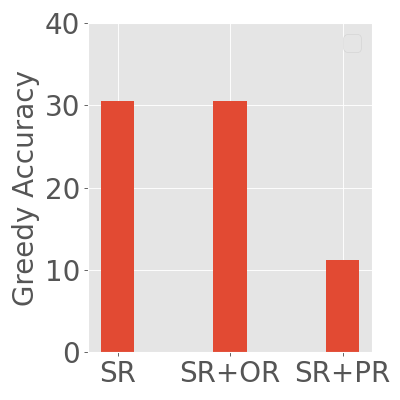}
    }
    \subfigure[Avg. Number of Tokens]{
        \label{fig:length-of-orm-prm}
        \includegraphics[height=35mm]{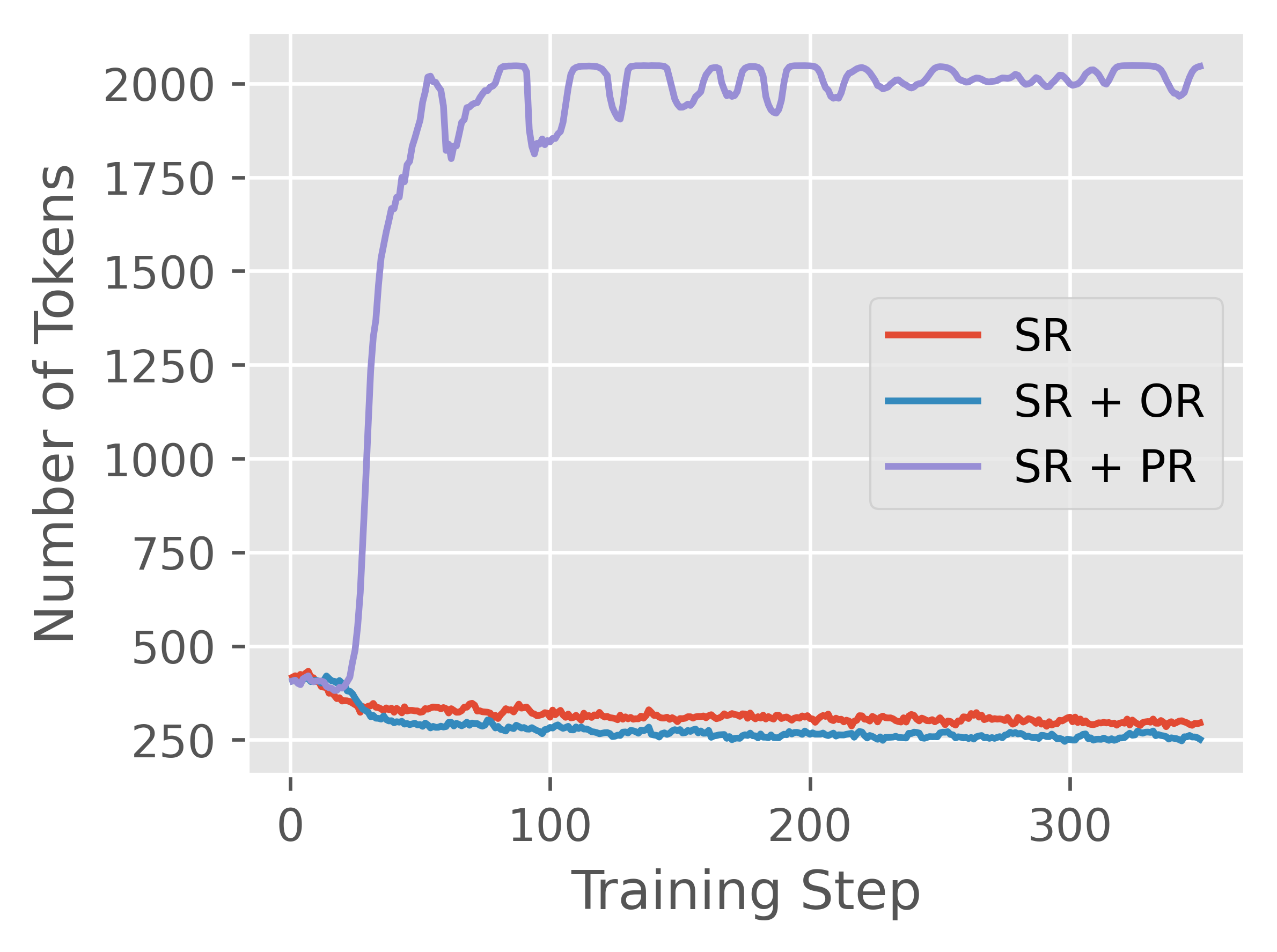}
    }
    \subfigure[Avg. Number of Steps
    ]{
        \label{fig:step-of-orm-prm}
        \includegraphics[height=35mm]{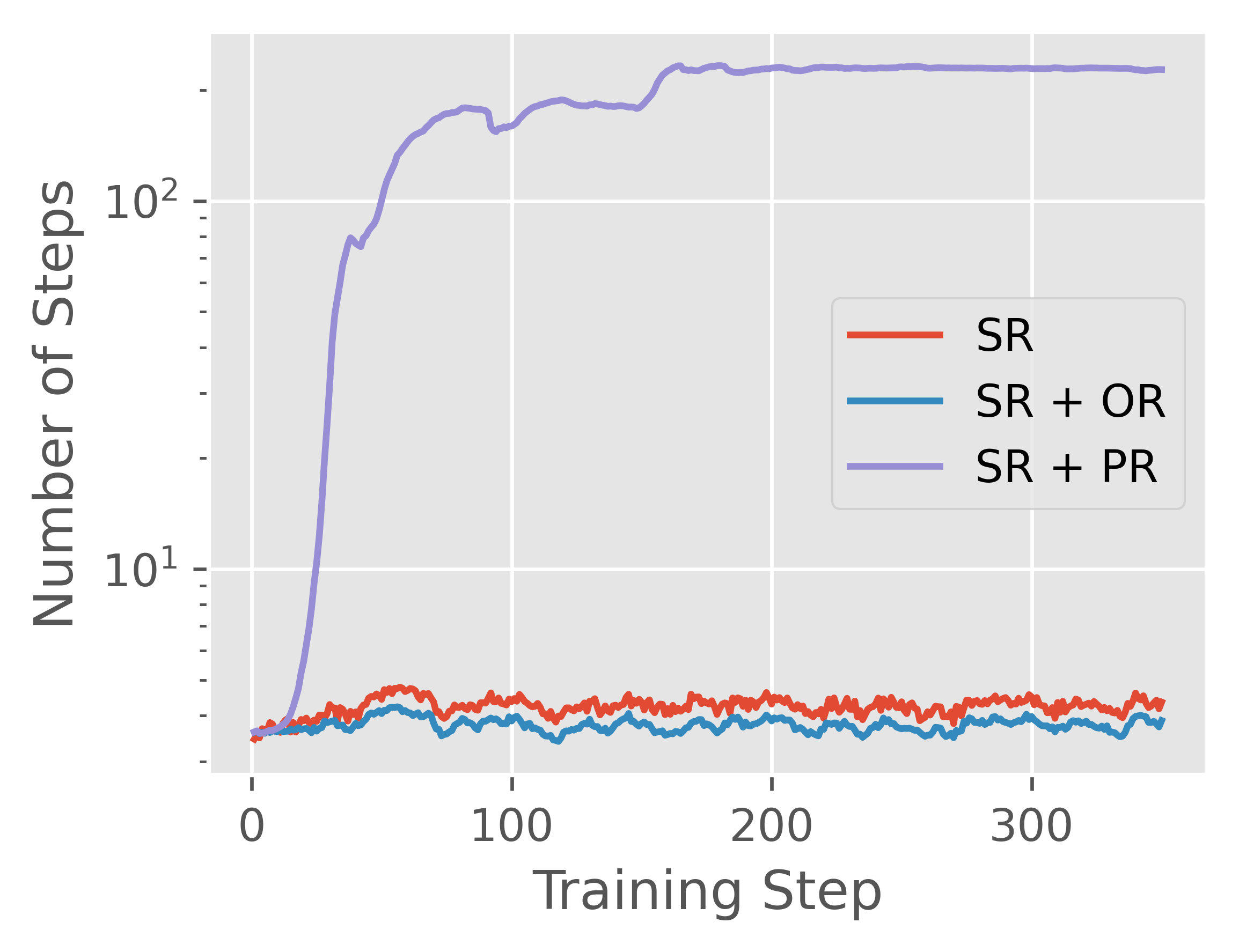}
    }
    
    \caption{(a) Greedy decoding accuracy of training with a combination of success reward and OR/PR on Qwen2-1.5B-Instruct. SR denotes the success reward. None of OR/PR can surpass training with success rewards. (b) Generation length during RL training. (c) Step count during RL training.}

\end{figure}

\paragraph{Results.} Surprisingly, we find these reward functions may not benefit RL training, yielding even worse inference-time performances than LLMs trained with a sparse success reward only, as shown in Fig.~\ref{fig:or-pr}. 
To further investigate the cause of performance degradation, 
Fig.~\ref{fig:length-of-orm-prm} reports the change in the generation length and the number of reasoning steps during training. 
Combining an outcome reward and a success reward shows similar training statistics and evaluation accuracy to adopting a sparse success reward only. We hypothesize this is because a success reward is accessible during training time, and an outcome reward may not be able to provide additional information beyond the success reward. 
On the other hand, when using PRM for RL training, we observe a significant change in the generation length and the number of reasoning steps during RL training. 
Specifically, the generation length and the step count of PR both significantly increase.

\begin{figure}
\centering    
        \label{fig:case-study-pr}
        \includegraphics[height=60mm]{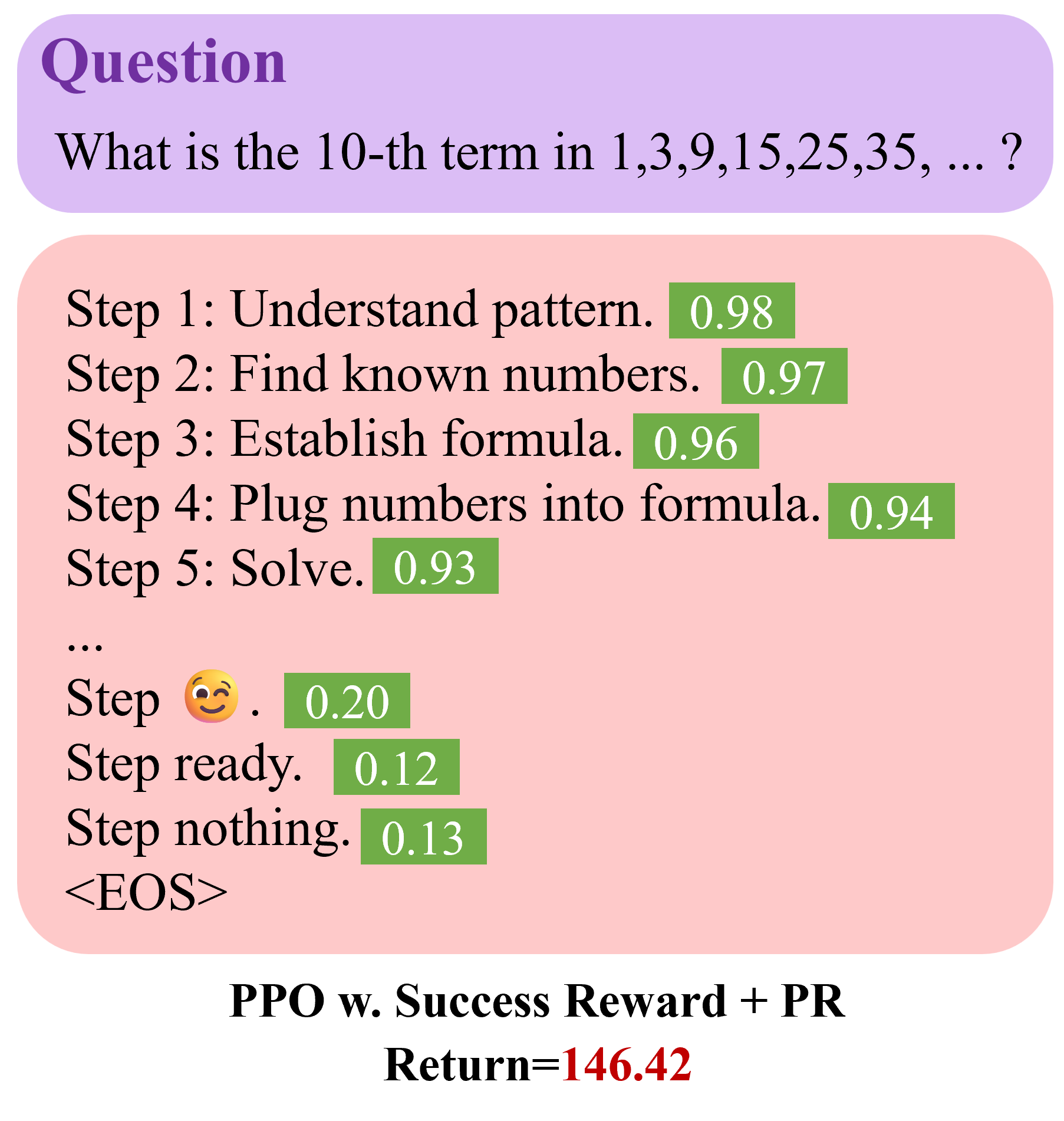}
    
\caption{{Case study of PR.} 
PRM provides rewards at the end of each step. 
 Through RL training with PR (Eq.~(\ref{eq:pr})), the LLM learns to generate many reasoning steps that do not contribute to problem-solving to achieve a high return.}
\label{fig:case-study}
\end{figure}

\paragraph{Case Study for PR.} For PR, a case study of the generated samples reveals the occurrence of the \emph{reward hacking} issue, that is, the LLM learns to obtain high rewards with some specific patterns without faithfully optimizing the ground-truth correctness through RL training.  In the generated solutions of PR, there are many short reasoning steps, but these steps only contain unnecessary or meaningless information that does not contribute to problem-solving. As the generation length increases, the model outputs only a single word or even emoji.

\begin{figure}
\centering     
    \subfigure[Repeat nonsense steps]{
        \label{fig:synthetic-nonsense}
        \includegraphics[height=33mm]{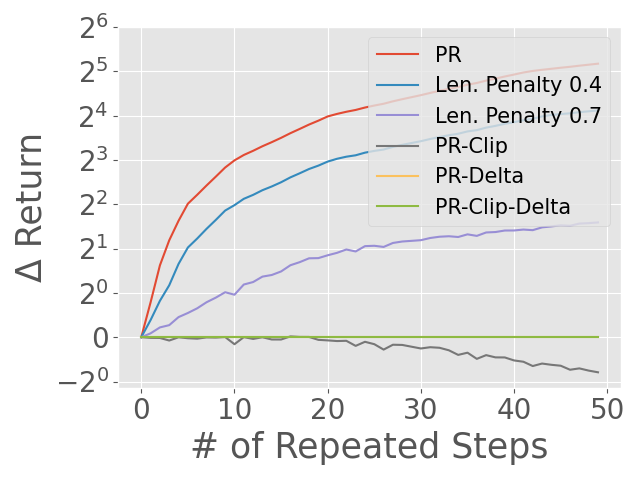}
    }
\hspace{-2mm}
    \subfigure[Repeat intermediate steps]{
        \label{fig:synthetic-mid-step}
        \includegraphics[height=33mm]{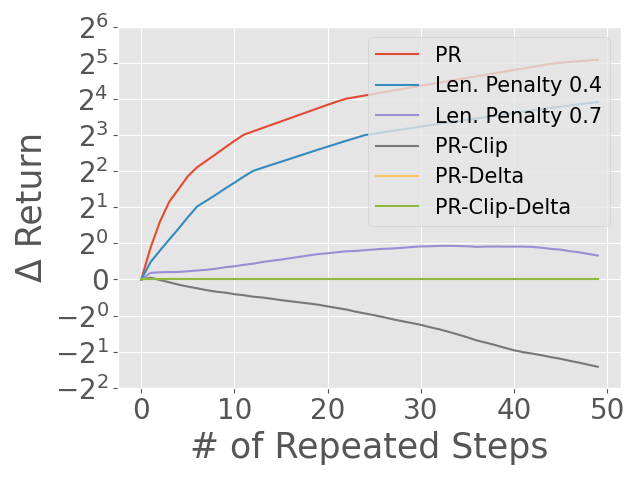} 
    }
    \hspace{-2mm}
    \subfigure[Repeat last sentences]{
        \label{fig:synthetic-last-sentence-length-norm}
        \includegraphics[height=33mm]{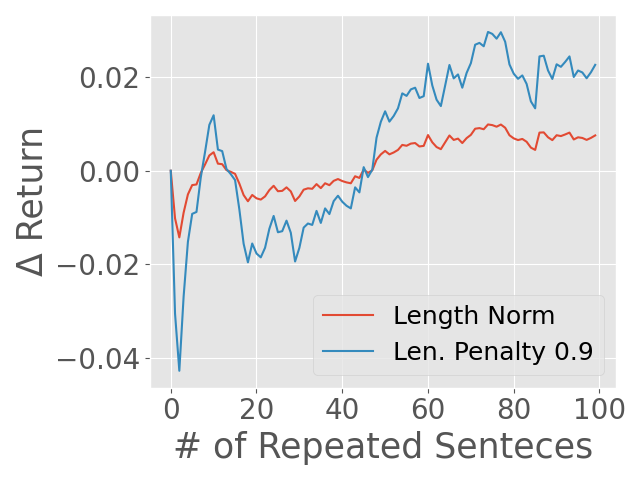} 
    }
\caption{Difference between the returns of synthetic solutions and the ground-truth solution. The synthetic solutions are constructed from the ground-truth solution by (a) repeatedly adding nonsense steps to the end of the sequence, (b) repeating an intermediate step, and (c) repeating the last sentence in the solution. 
A positive return difference indicates the repetitive patterns are favored over the ground-truth solution.
Both PR and length penalty can gain significantly high returns in (a) and (b). In (c), when sufficient repetitions are inserted, length normalization and length penalty would assign a higher return to the synthetic solution than to the ground-truth solution.}
\label{fig:synthetic}
\end{figure}

\paragraph{Analysis for PR.} The rewards of unnecessary reasoning steps are positive and could even be large, as shown in the case study (Fig.~\ref{fig:case-study}). The LLM learns to exploit this phenomenon by generating more reasoning steps, resulting in a higher return.
We further confirm the reward hacking behavior through some synthetic reasoning trajectories (Fig.~\ref{fig:synthetic-nonsense} and Fig.~\ref{fig:synthetic-mid-step}), where PR demonstrates extremely larger returns. This indicates that the PRM cannot effectively classify meaningless repetition as poor, which encourages the LLM to favor these unproductive steps. We observe two key properties when combining PR and the success reward for RL training,
\begin{itemize}
\item \textbf{The LLM learns to identify reasoning steps that yield high rewards but do not contribute to problem-solving.}  Specifically, reasoning steps that contain meaningless or unnecessary information can gain high rewards. 
\item \textbf{The RL objective can be optimized with simple patterns that do not improve the overall accuracy.}  For PR, infinitely high returns can be achieved by generating more unnecessary reasoning steps. However, the addition of unnecessary reasoning steps can not guide the LLM to improve accuracy.
\end{itemize}

\paragraph{Takeaways.} Here are two key takeaways regarding the impact of applying ORM and PRM in RL training,
\begin{itemize}
\item For ORM, it does not improve over the sparse success reward.
We hypothesize this is because, when a success reward is available during training time, ORM does not provide additional supervision signal and should not be a preferred choice at RL training time.
\item PRM would lead to a severe reward hacking issue during RL training due to repetition. Although PRM provides useful training signals, it is critical to prevent reward hacking.
\end{itemize}


\subsection{Techniques for Mitigating Reward Hacking}

\begin{figure}[h]
    \centering
    \includegraphics[width=0.9\linewidth]{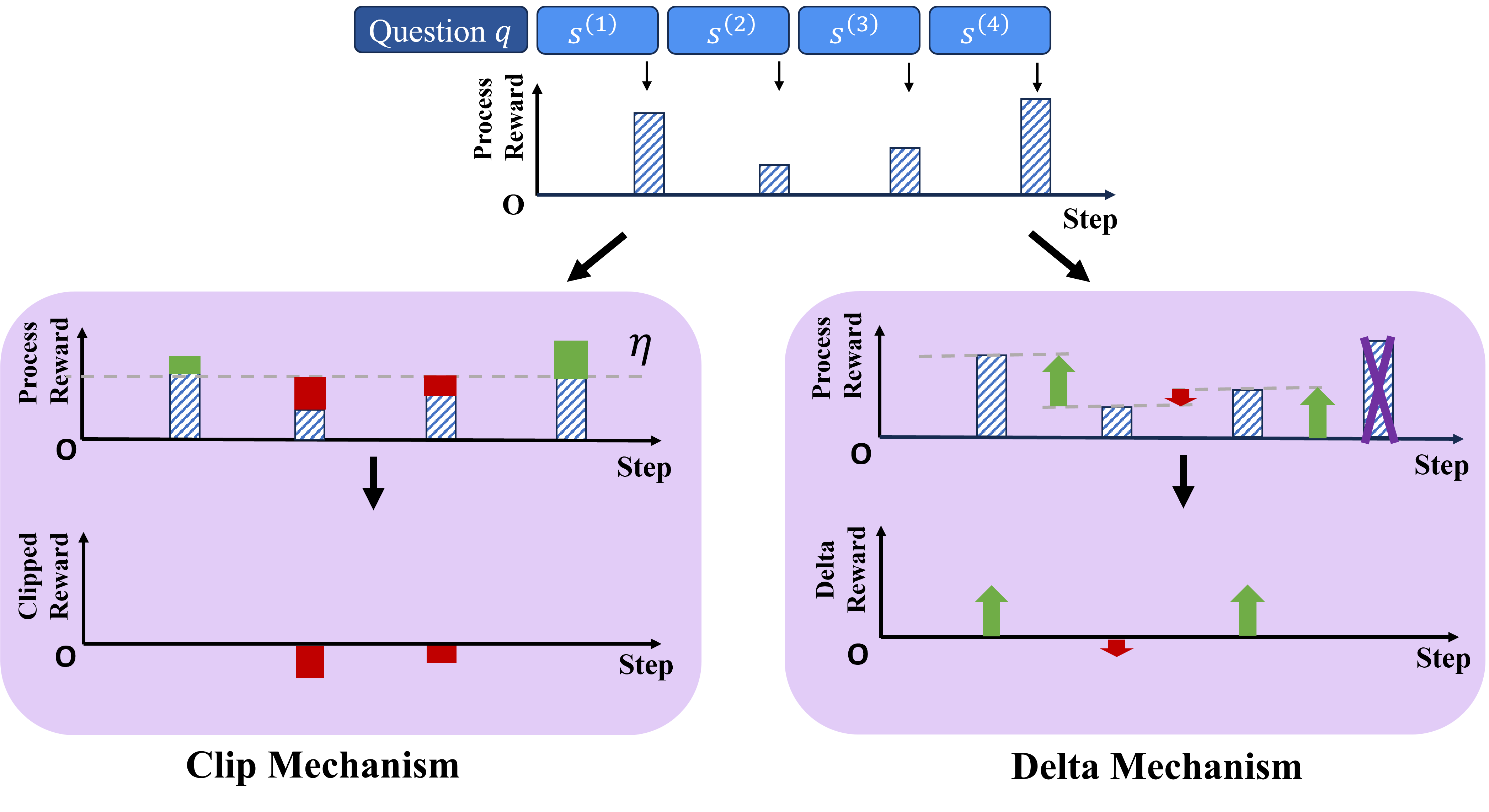}
    \caption{The \emph{Clip} mechanism and the \emph{Delta} mechanism. The Clip mechanism subtracts the rewards with a suitable threshold $\eta$ and upper-bounds all rewards with zero (Eq.~(\ref{eq:clip})). The Delta mechanism drops the last-step reward and computes the difference of rewards between two adjacent steps (Eq.~(\ref{eq:delta})). These mechanisms can alleviate the reward hacking issue of PRM in RL training.}
    \label{fig:mechanisms}
\end{figure}

Since ORM does not provide dense feedback for RL training and may lack additional information beyond the success reward during training, PRM can be a more suitable source for dense rewards. However, as analyzed in Sec.~\ref{sec:common}, PRM 
may enable an LLM to achieve an excessively high return by repeating unnecessary reasoning steps. 
To maintain a bounded objective while leveraging the ability of PRM to promote better reasoning skills, we introduce two novel techniques designed to utilize PRM in RL training effectively,
\begin{itemize}
\item \textbf{\emph{Clip} mechanism.} To prevent the LLM from exploiting the reward model by repetition, a straightforward idea is to upper-bound high rewards. Specifically, rewards $r_{\text{process}}$ are upper-bounded by a selected threshold $\eta$. We further ensure the return of a trajectory is bounded by subtracting all rewards by $\eta$. Formally, with a threshold $\eta$, 
\begin{align}
\label{eq:clip}
r(q,p^{(k)})=\min(r_{\text{process}}(q, p^{(k)})-\eta, 0)
\end{align}
If a suitable $\eta$ is chosen, the majority of the reasoning steps would receive a reward of 0, and only steps with low $r_{\text{process}}$ would have a negative reward.

\item \textbf{\emph{Delta} mechanism.} Alternatively, \emph{Delta} mechanism can effectively upper-bound the RL objective during training by subtracting the rewards between adjacent steps. For a solution, the reward for the last reasoning step is dropped since the success reward would be sufficient to provide guidance for the last reasoning step. Formally, for a solution prefix $p^{(k)}$,

\begin{align}
\label{eq:delta}
r(q,p^{(k)})=\begin{cases}r_{\text{process}}(q,p^{(k)})-r_{\text{process}}(q,p^{(k+1)})&\text{        if }k<K-1\\r_{\text{process}}(q,p^{(k)})&\text{              if }k=K-1\\
0&\text{              if }k=K\end{cases}
\end{align}

A nice property of the Delta mechanism is that it ensures the return of a solution is $\alpha\cdot r_{\text{process}}(q,s^{(1)}) +\text{Correct}(q,s)$, which is bounded since the maximum output value of a PRM is $1$. Furthermore, the return starting from any intermediate solution step $p^{(k)}$ is $\alpha\cdot r_{\text{process}}(q,p^{(k)}) +\text{Correct}(q,s)$, which is unaffected by the process rewards of future steps. Further analysis is provided in Appendix.~\ref{app:theory}.
\end{itemize}

Both the Clip and Delta mechanisms can be used individually or in combination. In practice, we consider three approaches incorporating these mechanisms:
\begin{enumerate}
\item \textbf{Process Reward with Clip mechanism (PR-Clip):} This applies the Clip mechanism.
\item \textbf{Process Reward with Delta mechanism (PR-Delta):} This employs the Delta mechanism.
\item \textbf{Process Reward with Clip \& Delta mechanism (PR-Clip-Delta):} The Clip mechanism is applied first, followed by the Delta mechanism.
\end{enumerate}

We further perform evaluation on synthetic solutions that exhibit repetitive patterns in different ways. As shown in Fig.~\ref{fig:synthetic-mid-step} and Fig.~\ref{fig:synthetic-nonsense}, the Clip mechanism and the Delta mechanism can both successfully limit the upper bound of the returns on these synthetic solutions. Additionally, the Clip mechanism imposes increasingly smaller returns as the length of the repetitive pattern grows.

\paragraph{Other Practices.} We also compare with some adopted practices to avoid reward hacking in prior works~\citep{singhal2023long}, including length normalization and length penalty. More details can be found in Appendix~\ref{app:rw-baselines}. Length normalization normalizes the rewards for each solution. Length penalty imposes a constant penalty for each step. As illustrated in Fig.~\ref{fig:synthetic}, imposing length penalty and length normalization could still favor the undesired repetition modes over correct solutions. We also investigate standard normalization for PRM as employed by \cite{shao2024deepseekmath}, which we find would lead to training instability. More details can be found in Sec.~\ref{sec:ep-ablation}.
\section{Experiments}
\label{sec:exp}

In this section, we perform full RL training with different reward designs to further examine how to ensure a learned reward model can be effective at training time. We will first illustrate our experiment setup in Sec.~\ref{sec:exp-setup}, then conduct ablation studies in Sec.~\ref{sec:ep-ablation} and finally present our main results on 1.5B\&7B models in Sec.~\ref{sec:exp-main}.

\subsection{Experiment Setup}
\label{sec:exp-setup}

\paragraph{Training Dataset.} We conduct RL training on the MathInstruct~\citep{yue2023mammoth} dataset. In particular, we only use the questions and the golden answers in the dataset while the provided solutions are not used for training.
To constitute the reward training dataset, we use Qwen2-7B-Instruct to sample 16 answers for each question in the training dataset and keep those questions that have both correct and wrong answers. To train an ORM, binary cross entropy loss is adopted. For PRM training, we follow \cite{wang2024math} to generate automatic process annotations by using Qwen2-7B-Instruct as the completer. Specifically, for each step in the generated samples, we use the completer to sample $8$ solutions starting from the solution prefix. This step is labeled as correct if any of these 8 solutions reaches final correctness.

\paragraph{Benchmarks \& Metrics.} We carry out our evaluation on the GSM8K~\citep{GSM8K} and MATH \citep{MATH} datasets. We ensure there is no data contamination issue, that is, the questions in the test sets do not appear in the training set. For evaluation metrics, we report the \emph{Greedy} and \emph{Sampling} scores, which correspond to the accuracy when adopting greedy decoding and sampling with temperature of 1 as generation strategies, respectively. To further understand the impact of RL, we also report Pass@16, which evaluates the probability a model can generate the correct answer out of 16 trials.

\paragraph{Base Models.} Our experiments are taken over a series of large language models from the Qwen2~\citep{qwen2} family and the state-of-the-art LLMs for mathematical reasoning, Qwen2.5~\citep{yang2024qwen2} family. Specifically, we use various 1.5B and 7B LLMs, including general and math-specific models. For general models, we consider Qwen2-1.5B-Instruct and Qwen2-7B-Instruct. For math-specific models, we consider Qwen2-Math-1.5B-Instruct, Qwen2.5-Math-1.5B-Instruct, Qwen2-Math-7B-Instruct and Qwen2.5-Math-7B-Instruct. Note that these LLMs already equip sufficient instruction following ability and we do not perform any further supervised fine-tuning. Lastly, the PRM is trained with the same base model as the actor model.



\paragraph{RL Training} We adopt the Proximal Policy Optimization (PPO) implementation of ReaLHF \citep{mei2024realhf}, which supports fine-tuning LLMs with dense rewards. 
Following prior practices~\citep{shao2024deepseekmath,xu2024dpo}, we adopt a large batch size and sample multiple solutions for each question within a batch. For 1.5B models, there are $1024$ questions, and $8$ solutions are sampled for each question in a batch, leading to a batch size of $1024\times 8$. For 7B models, the batch size is $4096 \times 8$.
\footnote{We also conduct an ablation study on PPO batch size in Appendix~\ref{sec:ablation_bsz}.}
Each training batch is split into 4 minibatches. We apply a KL penalty coefficient of 0.1, a coefficient of 1 for dense rewards, and a coefficient of 5 for successful rewards. The learning rates of 1B and 7B actor models are 1e-6 and 1e-5, respectively, while all critic models use a learning rate of 5e-6. We use Adam optimizer weight decay of $0.05$. 
The 1.5B models are trained on a cluster of 4 machines, each with 8 Nvidia H100 GPUs, for approximately 8 hours. The 7B models are trained on a cluster of 8 machines, each with 8 Nvidia H100 GPUs, for approximately 20 hours.

\subsection{Ablation Study}
\label{sec:ep-ablation}

 \begin{table}[]
    \centering
    \begin{tabular}{c|cc}
    \toprule
         Method & Greedy & Sampling\\
    \midrule
         Qwen2-1.5B-Instruct & 24.90 & 16.79 \\
    \midrule
         Success Reward & 30.58 & 27.05 \\
         SR + OR & 30.57 & 27.12 \\
         SR + PR \textbf{(E1)} & 11.16 & 14.68 \\
         SR + PR-Normed \textbf{(E2)} & 29.66 & 27.14 \\
         SR + PR-Normed \textbf{(E5)} & 12.36 & 12.84 \\
    \midrule
         SR + PR-Clip & 30.30 & \textbf{28.40} \\
         SR + PR-Delta & 30.68\ & 27.96 \\
         SR + PR-Clip-Delta & \textbf{31.44} & 28.20 \\
    \bottomrule
    \end{tabular}
    \vspace{2mm}
    \caption{Ablation study of various reward functions on MATH with Qwen2-1.5B-Instruct. The results are tested on the MATH test set using greedy decoding and sampling.
    We train the base models for 5 epochs. For OR, PR-Clip, PR-Delta, and PR-Clip-Delta, we report the accuracy of the final model. For PR and PR-Normed, significant performance degradation happens in later epochs and thus we report the performance in early epochs. Here, \textbf{E1} denotes the results of the 1-st epoch.}
    \label{tab:ablation}
 \end{table}

\paragraph{The Clip Mechanism \& The Delta Mechanism} Our ablation study of the Clip mechanism and the Delta mechanism is presented in Table~\ref{tab:ablation}. We also consider a standard normalization variant of PR~\citep{shao2024deepseekmath}, denoted as PR-Normed. PPO training with OR can not surpass training with a sparse success reward. PR demonstrates severe performance degradation during training due to the reward hacking issue discussed in Sec.~\ref{sec:reward-hacking}. Similarly, the performance of PR-Normed also decreases in the latter epochs. Consequently, none of OR, PR, and PR-Normed can achieve higher greedy decoding accuracy than training with a success reward. On the other hand, the Delta mechanism successfully stabilizes RL training, surpassing training with a success reward. Finally, by combining the Clip mechanism and the Delta mechanism, PR-Clip-Delta demonstrates the best greedy decoding accuracy.

\paragraph{Effect of PR-Clip-Delta} We compare the performance improvements of PPO training over the base LLMs when using a success reward and additionally using PR-Clip-Delta as dense rewards in Fig.~\ref{fig:perf-improve}. In addition to Greedy and Sampling scores, we also consider the Pass@16 score, which we believe can roughly estimate the upper bound of the model's capacity. Using PR-Clip-Delta as dense rewards can consistently improve RL training, across all LLMs and all evaluation metrics, except the greedy decoding accuracy on Qwen2-Math-7B-Instruct. This suggests that applying the Clip mechanism and the Delta mechanism can effectively utilize the PRM to guide the LLM in learning better reasoning skills during RL training. We report the detailed numbers in Appendix~\ref{sec:addition_result}.

\subsection{Main Results}
\label{sec:exp-main}

\begin{table}[t]
\centering
\begin{tabular}{c|cccc}
\toprule
Model  & \multicolumn{2}{c}{GSM8K} & \multicolumn{2}{c}{MATH} \\
& Greedy & Sampling & Greedy & Sampling \\
\midrule
GPT-4o-2024-08-06 & 92.9 & - & 81.1 & - \\
\midrule
DeepSeekMath-7B-RL & 88.2 & - & 52.4 & - \\
Internlm2-math-plus-7B & 84.0 & - & 54.4 & - \\
Mathstral-7B-v0.1  & 84.9 & - &  56.6 & - \\
NuminaMath-7B-CoT & 75.4 & - &  55.2 & - \\
Llama-3.1-8B-Instruct  & 76.6 & - & 47.2 & - \\
\midrule
\multicolumn{5}{c}{1.5B Models} \\
\midrule
Qwen2-1.5B-Instruct &  50.19 & 44.58 & 24.90 & 16.79 \\
\rowcolor{bg!70}
+ PPO w. (SR + PR-Clip-Delta) & 68.76\bbonus{18.57} & 66.19\bbonus{21.61} & 31.44\bbonus{6.54} & 28.20\bbonus{11.41} \\

\midrule
Qwen2-Math-1.5B-Instruct & 83.62 & 81.50 & 69.98 & 64.51 \\
\rowcolor{bg!70}
+ PPO w. (SR + PR-Clip-Delta) & 85.67\bbonus{2.05} & 84.76\bbonus{3.26} & 70.94\bbonus{0.96} & 68.13\bbonus{3.62} \\
\midrule
Qwen2.5-Math-1.5B-Instruct  & 85.14 & 82.11 & 76.00 & 72.05\\
\rowcolor{bg!70}
+ PPO w. (SR + PR-Clip-Delta) & 87.34\bbonus{2.20} & 85.97\bbonus{3.86} & 76.78\bbonus{0.78} & 74.63\bbonus{2.58} \\
\midrule
\multicolumn{5}{c}{7B Models} \\
\midrule
Qwen2-7B-Instruct & 86.88  & 80.44 &  57.54 & 48.27 \\
\rowcolor{bg!70}
+ PPO w. (SR + PR-Clip-Delta) & 87.64\bbonus{0.76}  &  87.34\bbonus{6.90} & 60.54\bbonus{3.00} & 58.17\bbonus{9.90} \\
\midrule
Qwen2-Math-7B-Instruct & 89.61 & 89.23 & 75.30 & 72.09  \\
\rowcolor{bg!70}
+ PPO w. (SR + PR-Clip-Delta) & 90.90\bbonus{1.29} & 90.14\bbonus{0.91} & 76.00\bbonus{0.70} & 74.09\bbonus{2.00} \\
\midrule
Qwen2.5-Math-7B-Instruct & 95.60 & 80.74 & 83.30 & 52.76 \footnotemark{} \\
\rowcolor{bg!70}
+ PPO w. (SR + PR-Clip-Delta) & 95.60$^{0.00}$ & 95.07 \bbonus{14.33} & 83.38\bbonus{0.08} & 81.22\bbonus{28.46} \\
\bottomrule

\end{tabular}
\vspace{2mm}
\caption{Greedy and Sampling scores on GSM8K and MATH benchmarks. {PPO training using sparse success rewards and PR-Clip-Delta as dense rewards consistently improve all evaluated LLMs, including the state-of-the-art 7B LLMs, Qwen2.5-Math-7B-Instruct.} For sampling decoding, we adopt the temperature of 1.0. }
\label{tab:main}
\end{table}


\paragraph{Main Results} Our main results are summarized in Table.~\ref{tab:main}. 
RL training consistently improves the performance of the base model across all the models we test, even on the state-of-the-art 1.5B model, Qwen2.5-Math-1.5B-Instruct, and 7B model, Qwen2.5-Math-7B-Instruct. 
For 1.5B models, Qwen2-1.5B-Instruct obtains the most significant performance improvement. Through RL training with PR-Clip-Deta as reward function, the best 1.5B model, Qwen2.5-Math-1.5B-Instruct achieves 87.34\% and 76.78\% greedy decoding accuracy on GSM8K and MATH benchmark respectively, indicating 2.20\% and 0.78\% improvement of accuracy over the base model.
For 7B models, building on the strongest 7B LLM, Qwen2.5-Math-7B-Instruct, RL training with dense reward further boosts the performance and achieves 95.6\% and 83.38\% greedy decoding accuracy on GSM8K and MATH benchmarks, respectively, surpassing several baselines. It is noteworthy that Qwen2.5-Math-7B-Instruct is already trained using RL, and our results indicate that RL with a carefully crafted dense reward can further enhance its performance, highlighting the effectiveness of PR-Clip-Delta.

\begin{figure}[t]
\vspace{-5mm}
\centering     
    \subfigure[$\Delta$ Greedy]{
        \label{fig:perf_gain_greedy}
        \includegraphics[height=42mm]{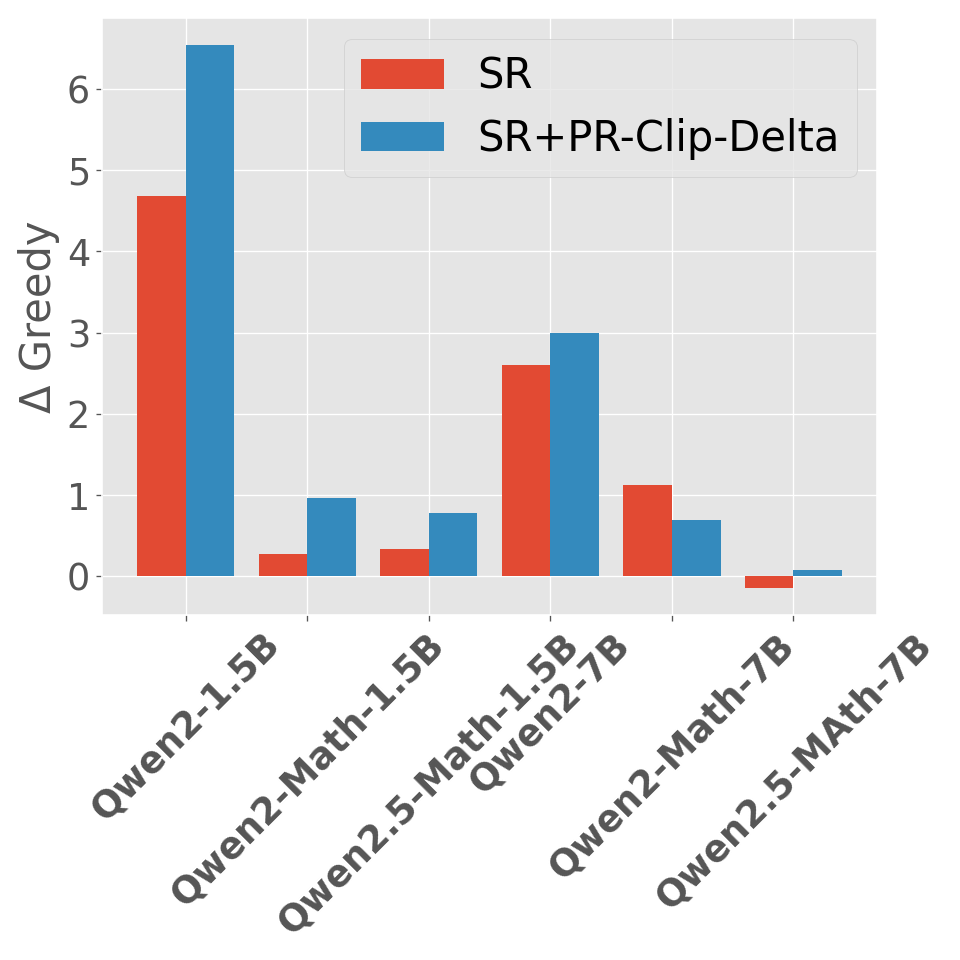}
    }
    \subfigure[$\Delta$  Sampling]{
        \label{fig:perf_gain_pass1}
        \includegraphics[height=42mm]{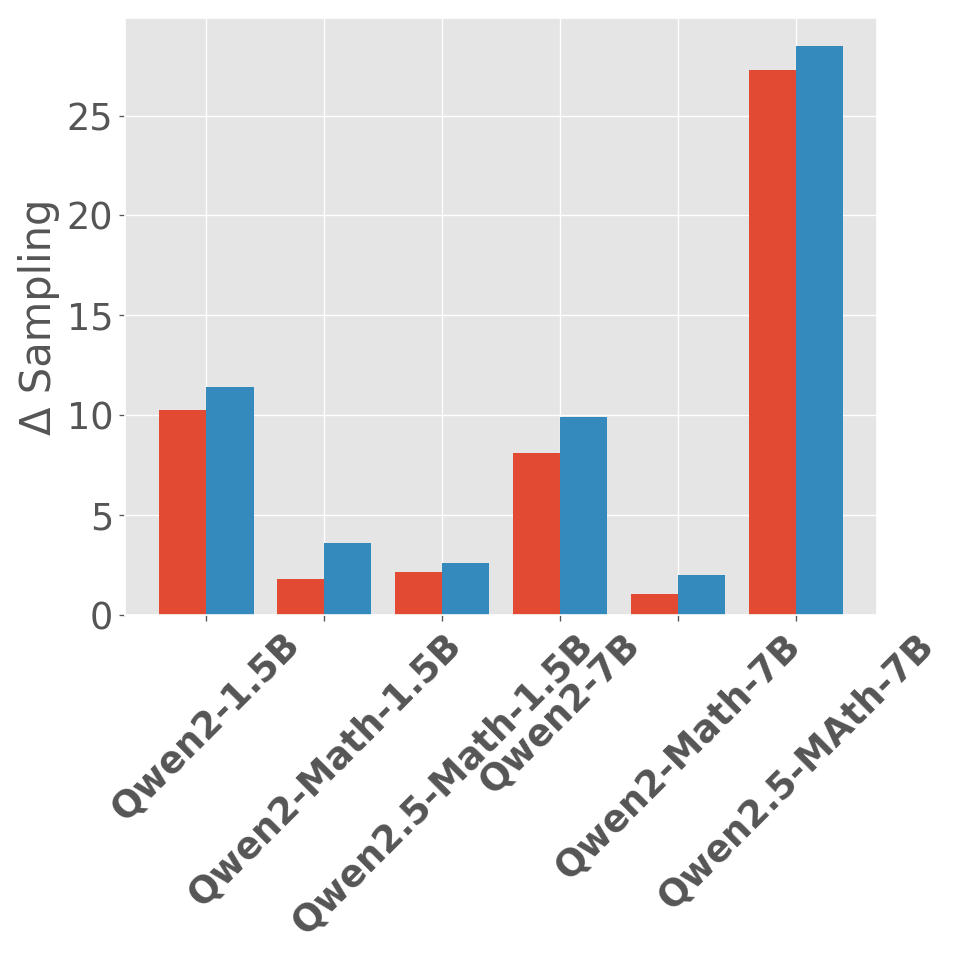} 
    }
    \subfigure[$\Delta$  Pass@16]{
        \label{fig:perf_gain_pass16}
        \includegraphics[height=42mm]{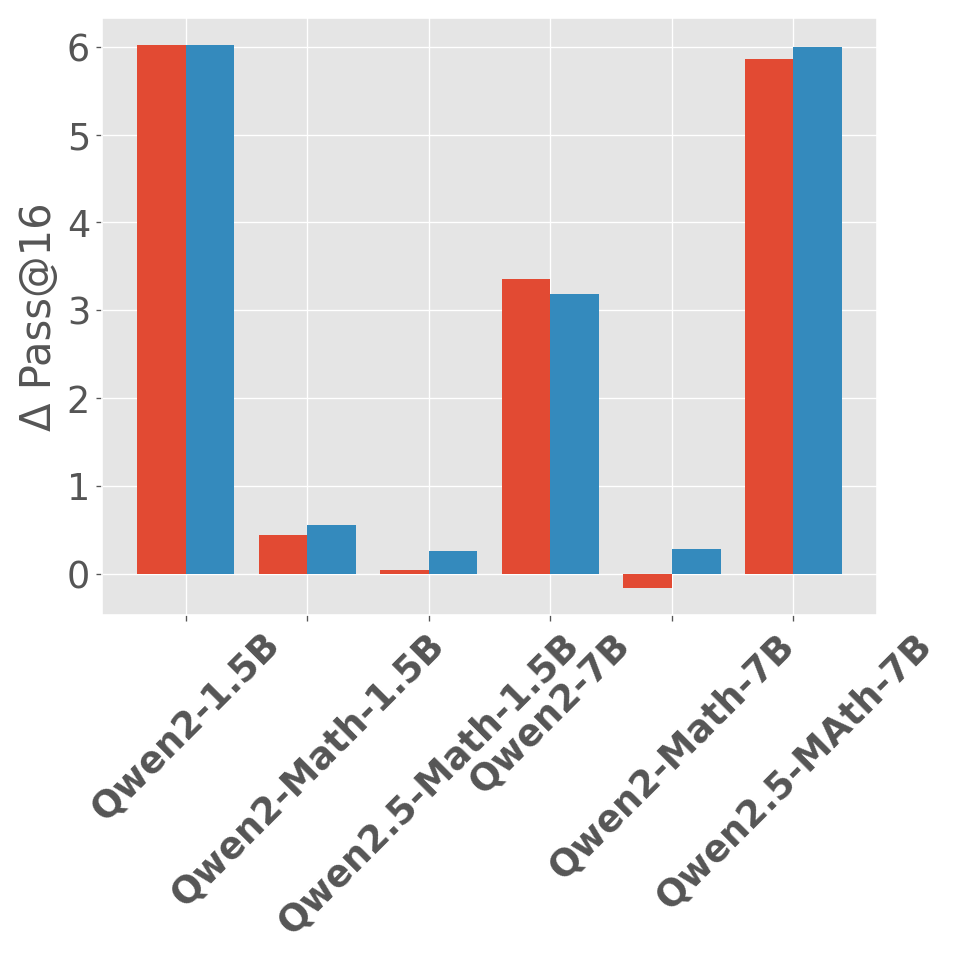}
    }
\caption{Performance improvement of PPO training over the base LLMs using success rewards and further using PR-Clip-Delta as dense rewards. All LLMs are the instruction following model and the "-Instruct" suffices are omitted for improved clarity. Adding PR-Clip-Delta as dense rewards consistently improve RL training with sparse success rewards only.} 
\label{fig:perf-improve}
\end{figure}

\paragraph{Performance Improvement} 
The performance improvement of RL training varies across models with different amounts of parameters and different strengths. In general, weaker models gain higher performance improvements than stronger models. Comparing the improvements of Greedy and Sampling scores, the improvements of Sampling score are larger than those of Greedy score across all LLMs, resulting in a smaller gap between Sampling and Greedy scores. Interestingly, we also highlight the comparison between Qwen2.5-1.5B-Instruct and Qwen2-7B-Instruct since both models have very close performance on MATH but have different amounts of parameters. The smaller 1.5B model, Qwen2.5-1.5B-Instruct, has a more significant improvement than and can surpass the larger 7B model, Qwen2-7B-Instruct, on both MATH and GSM8K benchmarks. 

\footnotetext{For sampling accuracy, we find that Qwen-2.5-math-Instruct is likely to generate strange characters, leading to poor sampling accuracy.}

\section{Conclusion}

In this work, we investigate designing dense rewards with a process-supervised reward model in RL training to improve the reasoning ability of LLMs. We examine some popular reward models and identify the issue of reward hacking, which manifests as the generation of nonsensical texts or unnecessary reasoning steps. The reward hacking issue can be mitigated with our proposed techniques, using the Clip mechanism to prevent the LLM from exploiting the reward model and the Delta mechanism to ensure a bounded RL objective. We show that the proposed techniques can be utilized to apply Process-supervised Reward Models for improved RL training. 

\paragraph{Limitations.} Limited by computation resources, our experiments are conducted over 1.5B\&7B LLMs, while evaluations on larger LLMs could further help verify our proposed techniques. Also, it is an interesting direction to perform various inference-time search strategies with the LLMs trained with PPO, which could help further understand whether RL training can improve search performance. Furthermore, we believe that with the support of more powerful reward models, RL training can bring greater benefits to LLM reasoning.

\bibliography{main}
\bibliographystyle{preprint}

\newpage
\appendix

\section{Additional results}
\label{sec:addition_result}

In Tab.~\ref{tab:sparse_vs_prm} and Tab.~\ref{tab:sparse_vs_prm_gsm}, we report the results of RL training on different base models, including those with success rewards and after applying PR-Clip-Delta.

\begin{table}[htb]
\centering
\begin{tabular}{@{}c|c|ccc@{}}
\toprule
Model                                & Method           & \multicolumn{3}{c}{Math}                         \\ \midrule
                              &                  & Greedy         & Sample         & Pass@16        \\ \midrule
\multirow{3}{*}{Qwen2-1.5B-Instruct}   & Basemodel    & 24.90 & 16.79 &  55.68       \\ \cmidrule(l){2-5} 
                                     & Success Reward & 30.58\bbonus{4.68} & 27.05\bbonus{10.26} & 61.70\bbonus{6.02}          \\
                                     & + PR-Clip-Delta &  31.44\bbonus{6.54} & 28.20\bbonus{11.41} & 61.70\bbonus{6.02}    \\ \midrule
\multirow{3}{*}{Qwen2-Math-1.5B-Instruct}   & Basemodel   & 69.98 & 64.51 & 88.02 \\ 
\cmidrule(l){2-5} 
                                     & Success Reward & 70.26\bbonus{0.28} & 66.29\bbonus{1.78} & 88.46\bbonus{0.44} \\
                                     & + PR-Clip-Delta &  70.94\bbonus{0.96} & 68.13\bbonus{3.62} & 88.58\bbonus{0.56} \\  \midrule

\multirow{3}{*}{Qwen2.5-Math-1.5B-Instruct}   & Basemodel   & 76.00 & 72.05 & 90.50 \\ \cmidrule(l){2-5} 
                                     & Success Reward & 76.34\bbonus{0.34} & 74.22\bbonus{2.17} & 90.54\bbonus{0.04} \\
                                     & + PR-Clip-Delta & 76.78\bbonus{0.78} & 74.63\bbonus{2.58} & 90.76\bbonus{0.26} \\  
\midrule
\midrule
\multirow{3}{*}{Qwen2-7B-Instruct}   & Basemodel    & 57.54 & 48.27 & 80.04          \\ \cmidrule(l){2-5} 
                                     & Success Reward & 60.14\bbonus{2.60} & 56.39\bbonus{8.12} & 83.40\bbonus{3.36}          \\
                                     & + PR-Clip-Delta &  60.54\bbonus{3.00} & 58.17\bbonus{9.90} & 83.22\bbonus{3.18}   \\ \midrule
\multirow{3}{*}{Qwen2-Math-7B-Instruct}   & Basemodel   & 75.30 & 72.09 & 91.24 \\ \cmidrule(l){2-5} 
                                     & Success Reward      & 76.42\bbonus{1.12} & 73.12\bbonus{1.03} & 91.08\ddrop{0.16}  \\
                                     & + PR-Clip-Delta & 76.00\bbonus{0.70} & 74.09\bbonus{2.00} & 91.52\bbonus{0.28} \\  \midrule

\multirow{3}{*}{Qwen2.5-Math-7B-Instruct}   & Basemodel   & 83.3 & 52.76 & 86.6 \\ \cmidrule(l){2-5} 
                                     & Success Reward      & 83.16\ddrop{0.14} & 79.95\bbonus{27.19} & 92.46\bbonus{5.86}  \\
                                     & + PR-Clip-Delta &
                                     83.38\bbonus{0.08} & 81.22\bbonus{28.46} & 92.60\bbonus{6.00}\\  \bottomrule
\end{tabular}
\caption{Results on MATH test set}
\label{tab:sparse_vs_prm}
\end{table}

\begin{table}[htb]
\centering
\begin{tabular}{@{}c|c|cc@{}}
\toprule
Model                                & Method           & \multicolumn{2}{c}{GSM8K}                         \\ \midrule
                              &                  & Greedy         & Sample                \\ \midrule
\multirow{3}{*}{Qwen2-1.5B-Instruct}   & Basemodel    & 50.19 & 44.58      \\ 
\cmidrule(l){2-4} 
                                     & Success Reward & 67.70\bbonus{17.51} & 65.50\bbonus{20.92} \\
                                     & + PR-Clip-Delta & 68.76\bbonus{18.57} & 66.19\bbonus{21.61}  \\
                                     \midrule
\multirow{3}{*}{Qwen2-Math-1.5B-Instruct}   & Basemodel   & 83.62 & 81.50  \\ 
\cmidrule(l){2-4} 
                                     & Success Reward & 84.61\bbonus{0.99} & 83.93\bbonus{2.43}  \\
                                     & + PR-Clip-Delta & 85.67\bbonus{2.05} & 84.76\bbonus{3.26}  \\  \midrule

\multirow{3}{*}{Qwen2.5-Math-1.5B-Instruct}   & Basemodel   & 85.14 & 82.11 \\ 
\cmidrule(l){2-4} 
                                     & Success Reward & 86.73\bbonus{1.59} & 85.82\bbonus{3.71}  \\
                                     & + PR-Clip-Delta & 87.34\bbonus{2.20} & 85.97\bbonus{3.86}  \\  
\midrule
\midrule
\multirow{3}{*}{Qwen2-7B-Instruct}   & Basemodel    & 86.88 & 80.44  \\ 
\cmidrule(l){2-4} 
                                     & Success Reward &  87.72\bbonus{0.84} & 86.81\bbonus{6.37} \\
                                     & + PR-Clip-Delta &  87.64\bbonus{0.76} &  87.34\bbonus{6.90} \\ 
                                    \midrule
\multirow{3}{*}{Qwen2-Math-7B-Instruct}   & Basemodel  & 89.61  & 89.23 \\ 
\cmidrule(l){2-4} 
                                     & Success Reward  & 89.46\ddrop{0.15} & 90.07\bbonus{0.84} \\
                                     & + PR-Clip-Delta & 90.90\bbonus{1.29} & 90.14\bbonus{0.91}  \\  \midrule

\multirow{3}{*}{Qwen2.5-Math-7B-Instruct}   & Basemodel   & 95.60 & 80.74 \\ 
\cmidrule(l){2-4} 
                                     & Success Reward      &  95.45\ddrop{0.15} & 95.07\bbonus{14.33} \\
                                     & + PR-Clip-Delta  & 95.60\bbonus{0.00} & 95.07\bbonus{14.33} \\  
                                \bottomrule
\end{tabular}
\caption{Results on GSM8K test set}
\label{tab:sparse_vs_prm_gsm}
\end{table}

In Fig.~\ref{fig:prm-normed_acc_curve}, we report the greedy accuracy on MATH test set of different training epochs, where epoch-0 means the base model (i.e., Qwen2-1.5B-Instruct). The introduction of PR-norm caused the model's accuracy to drop significantly starting from the third epoch.

\begin{figure}[ht]
    \centering
    \includegraphics[width=0.5\linewidth]{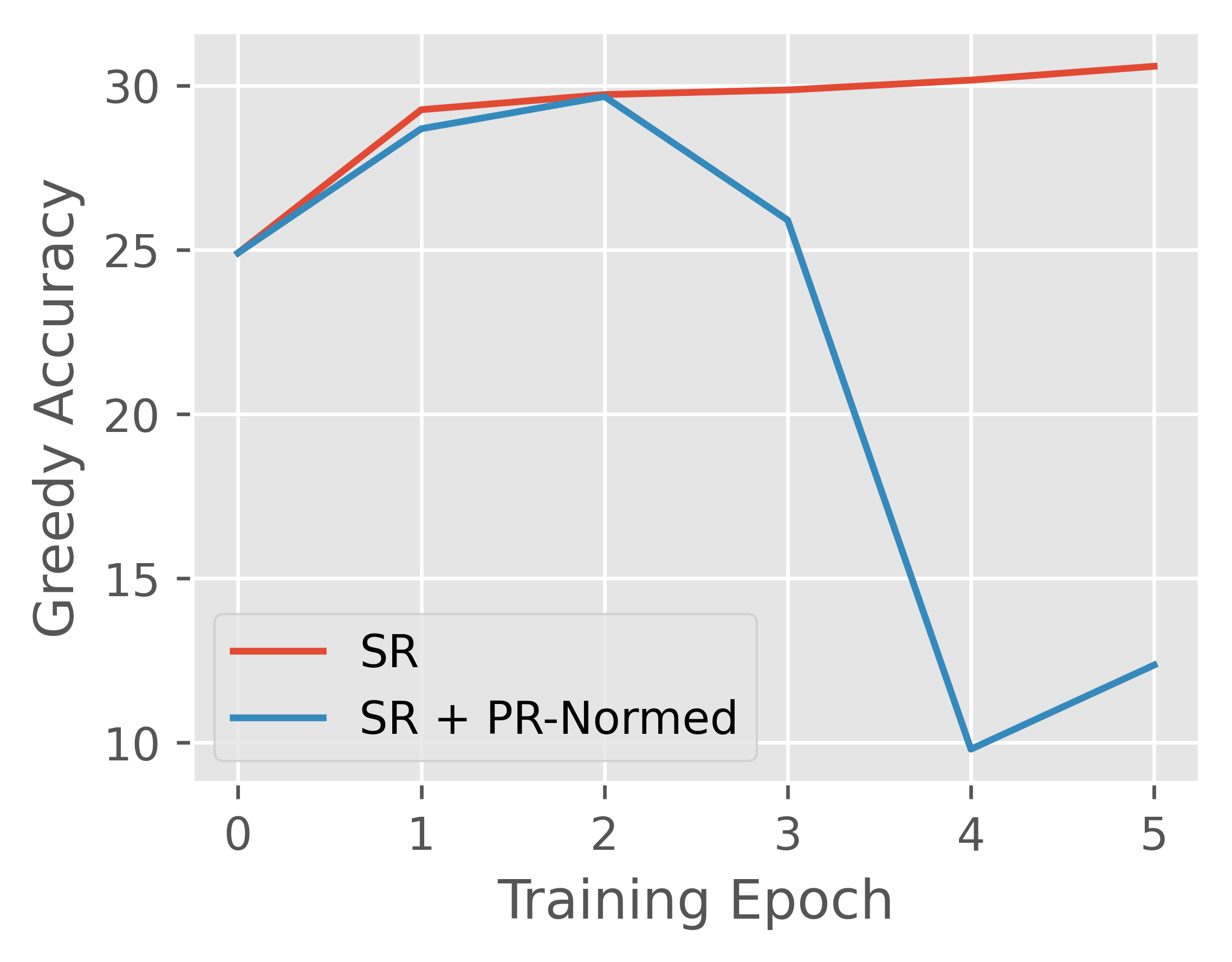}
    \caption{Greedy accuracy on MATH test set during the training process.}
    \label{fig:prm-normed_acc_curve}
\end{figure}

\section{Ablation Study on PPO Batch Size}
\label{sec:ablation_bsz}
We conduct an ablation study on PPO batch size in Tab.~\ref{tab:ablation_bsz} and Tab.~\ref{tab:ablation_bsz_2}, where $a \times b$ denotes $a$ questions with $b$ solutions sampled per question within each batch. The results indicate that larger final batch sizes improve performance. With a fixed final batch size, configurations of $1024 \times 8$ and $2048 \times 4$ achieve higher accuracy compared to $8192 \times 1$. In practice, we recommend sampling approximately 8 solutions per question while increasing the number of questions in each batch.
\begin{table}[]
    \centering
    \begin{tabular}{c|ccc}
    \toprule
        PPO Batch Size & Greedy & Sample & Pass@16\\
    \midrule
        Qwen-2-1.5B-Instruct & 24.90 & 16.79 & 55.68 \\
    \midrule
        $1024 \times 1$ & 29.12 & 24.68 & 60.06 \\
        $2048 \times 1$ & 29.00 & 25.43 & 60.16 \\
        $8192 \times 1$ & 30.12 & 25.82 & 61.06 \\
    \midrule
        $2048 \times 4$ & \textbf{30.66} & 26.63 & 61.32 \\
        $1024 \times 8$ & 30.58 & \textbf{27.05} & \textbf{61.70}\\
    \bottomrule
    \end{tabular}
    \caption{The ablation study on PPO batch size is conducted using Qwen2-1.5B-Instruct and the MATH, where $a \times b$ represents $a$ questions with $b$ solutions sampled per question in each batch. All models are trained using only the Success Reward.}
    \label{tab:ablation_bsz}
\end{table}

\begin{table}[]
    \centering
    \begin{tabular}{c|ccc}
    \toprule
        PPO Batch Size & Greedy & Sample & Pass@16\\
    \midrule
        Qwen2-7B-Instruct & 57.54 & 48.27 & 80.04 \\
    \midrule
        $1024 \times 8$ & 59.06 & 56.35 & 82.32 \\
        $4096 \times 8$ & \textbf{60.14} & \textbf{56.39} & \textbf{83.40}\\
    \bottomrule
    \end{tabular}
    \caption{The ablation study on PPO batch size is conducted using Qwen2-7B-Instruct and the MATH. All models are trained using only the Success Reward.}
    \label{tab:ablation_bsz_2}
\end{table}

\newpage
\section{Synthetic Example}
In Fig.~\ref{fig:synthetic}, we insert nonsense steps or repeat the last sentences in the following manner. 
\begin{mdframed}
\textbf{Prompt:}

Tom found a box of 220 baseball cards. Being a thoughtful older brother, he wanted to save at least 50 percent of the cards for himself and split the remaining cards evenly among his 5 younger brothers. What is the lowest possible number of cards that Tom can save for himself that will also allow each of his 5 younger brothers to get an equal number of cards?
Please reason step by step, and put your final answer within $\backslash$boxed\{\}.
\\

\textbf{Answer by Qwen2-1.5B-instruct: }

Step 1: 50\% of 220 cards
$$
50\% \times 220 = 0.50 \times 220 = 110
$$

Therefore, Tom would save 110 of the cards for himself.

Step 2: Remaining cards that Tom would like to distribute evenly among his 5 younger brothers

$$
220 - 110 = 110
$$

Tom would have 110 cards left to distribute evenly among his 5 younger brothers.

Step 3: Forming sets of cards and evenly distributing

If the 110 remaining cards are to be divided into 5 sets of equal cards, each set must contain:

$$
\frac{110}{5} = 22
$$

Therefore, each of Tom's 5 younger brothers would get 22 cards, but since each set also must have 22 cards, the lowest possible number of cards for Tom to save so that each brother gets an equal number of cards means Tom would save:

$$
5 \times 22 = 110
$$

The lowest possible number of cards that Tom can save for himself that will allow each of his 5 younger brothers to get an equal number of cards is 110. This way, Tom will save enough cards to permit a remainder that will also allow one card to be shared by each of his brothers out of the total of 110 cards.

The answer is: $\boxed{110}$.
\\

\textbf{Repeat Nonsense}

Step 1: 50\% of 220 cards
$$
50\% \times 220 = 0.50 \times 220 = 110
$$
...

Step 3: Forming sets of cards and evenly distributing
...

\textcolor{red}{Step done.}

\textcolor{red}{Step done.}

\textcolor{red}{Step end.}

\textcolor{red}{Step end.}

\textcolor{red}{Step show.}

...
\\

\textbf{Repeat Mid Step}

Step 1: 50\% of 220 cards
$$
50\% \times 220 = 0.50 \times 220 = 110
$$
...

\textcolor{red}{
Step 2: Remaining cards that Tom would like to distribute evenly among his 5 younger...
}

\textcolor{red}{
Step 2: Remaining cards that Tom would like to distribute evenly among his 5 younger...
}

\textcolor{red}{
Step 2: Remaining cards that Tom would like to distribute evenly among his 5 younger...
}

...

Step 3: Forming sets of cards and evenly distributing
...

The answer is: $\boxed{110}$.
\\

\textbf{Repeat Last Sentence}

Step 1: 50\% of 220 cards
$$
50\% \times 220 = 0.50 \times 220 = 110
$$
...

Step 3: Forming sets of cards and evenly distributing
...

\textcolor{red}{The answer is: $\boxed{110}$.}

\textcolor{red}{The answer is: $\boxed{110}$.}

\textcolor{red}{The answer is: $\boxed{110}$.}

\textcolor{red}{The answer is: $\boxed{110}$.}

...
\end{mdframed}

\section{Baselines}
\label{app:rw-baselines}

\paragraph{Length Normalization.} Length normalization normalizes the rewards for each solution. Formally,
\begin{align*}
r(q, p^{(k)})=\frac{1}{K}r_{\text{process}}(q, p^{(k)})
\end{align*}

\paragraph{Length Penalty.} Length penalty imposes a constant penalty for each step.
\begin{align*}
r(q, p^{(k)})=r_{\text{process}}(q, p^{(k)}) - k * c_\text{penalty}
\end{align*}

\subsection{Theoretical Analysis}
\label{app:theory}

\begin{theorem} By applying the Delta mechanism to the process rewards, the return for any token in step $p^{(k)}$ is,
$$
\text{Return}(q, p^{(k)})=\begin{cases}\alpha\cdot r_{\text{process}}(q,p^{(k)})+\text{Correct}(q,s)&\text{if }k<K\\\text{Correct(q,s)}&\text{otherwise}\end{cases}
$$

\end{theorem}

\begin{proof}
For $k\le K-1$,
\begin{align*}
\text{Return}&(q,p^{(k)})=\alpha\cdot\left(\sum_{i=k}^{K-2}r_{\text{process}}(q, p^{(i)})-r_{\text{process}}(q, p^{(i+1)})+r_{\text{process}}(q,p^{(K-1)})\right)+\text{Correct(q,s)}\\
&=\alpha\cdot\left(r_{\text{process}}(q, p^{(k)})-r_{\text{process}}(q, p^{(k+1)})+r_{\text{process}}(q, p^{(k+1)})-r_{\text{process}}(q, p^{(k+2)}) + \cdots \right. \\ 
& \left. -r_{\text{process}}(q, p^{(K-2)})+r_{\text{process}}(q, p^{(K-2)})-r_{\text{process}}(q, p^{(K-1)})+r_{\text{process}}(q,p^{(K-1)})\right)+\text{Correct(q,s)}\\
&=\alpha\cdot r_{\text{process}}(q,p^{(k)})+\text{Correct}(q,s)
\end{align*}

For $k=K$, the return is $\text{Correct}(q,s)$ clearly.

\end{proof}

This result indicates that when applying the Delta mechanism to the process rewards, the policy gradient for optimizing the policy $\pi$ would be,
\begin{align*}
\nabla_\pi \mathcal L_{\text{RL}}(\pi)=\sum_{k=1}^K\nabla_\pi\log\pi(s^{(k)}|q,p^{(k-1)})\cdot(\alpha\cdot r_{\text{process}}(q, p^{(k)})+\text{Correct}(q,s))
\end{align*}

\end{document}